\documentclass{article}
\usepackage{graphicx}
\usepackage{hyperref} 

\usepackage{amsmath, amsthm, amssymb}
\usepackage{tikz}
\usepackage{subcaption}
\usetikzlibrary{shapes.geometric}
\usetikzlibrary{angles,quotes}
\usetikzlibrary{calc} 
\usetikzlibrary{arrows.meta, positioning}
\usepackage{cleveref}

\usepackage{tcolorbox}
\tcbuselibrary{skins}
\usepackage{dashrule} 
\usepackage{xcolor}
\usepackage{wrapfig}
\usepackage{fullpage}
\usepackage{natbib}

\usepackage{booktabs}
\usepackage{nicematrix}

\newtheorem{theorem}{Theorem}

\newtheorem{lemma}[theorem]{Lemma}

\newtheorem{corollary}[theorem]{Corollary}

\newtheorem{remark}[theorem]{Remark}

\title{Parity, Sensitivity, and Transformers}

\author{Alexander Kozachinskiy\\ CENIA\\ \texttt{alexander.kozachinskyi@cenia.cl} \and Tomasz Steifer\\IPPT PAN\\\texttt{tsteifer@ippt.pan.pl} \and Przemys{\l}aw Andrzej Wa{\l}\c{e}ga \\ Queen Mary University of London \\\texttt{p.walega@qmul.ac.uk}}

\begin{document}
	\maketitle
	
	\begin{abstract}%
		
		Understanding what neural architectures can and cannot compute is a central challenge in the theory of AI. One of the fundamental problems in this context is the PARITY task, which asks whether the number of 1s in a binary input sequence is even or odd. PARITY is one of the central tasks studied in the theory of computation, yet it remains surprisingly unclear under which conditions transformers can or cannot solve it.
		
		In this paper, we show that the minimal number of layers a transformer needs to compute PARITY is two. In particular, we solve the open problem asking whether a one-layer transformer can compute PARITY.  
		We answer it negatively by showing  that average sensitivity of a one-layer transformer grows slower than that of PARITY.
		Furthermore, we show a new construction for transformer that computes PARITY, which improves on the existing constructions by removing a number of impractical assumptions.
		In particular, the existing transformers for PARITY rely on such impractical assumptions as  length-dependent positional encoding, hardmax, layernorm without a regularisation parameter, or incompatibility with causal masking. We show that these assumptions can be removed, at the cost of increasing the number of layers from two  to four. Specifically, we show that PARITY can be computed by a four-layer transformer  using softmax attention, length-independent and polynomially bounded positional encoding, no layernorm, and compatible with both causal and non-causal masking.
		These results advance our understanding of the computational capabilities of transformer architectures by precisely characterizing the depth requirements for PARITY under realistic design constraints. 
	\end{abstract}

	\section{Introduction}
	Imagine that you are trying to train a neural network---or a related model such as transformer, Graph Neural Network etc.---on some task and you see that the accuracy oscillates and never reach a reasonable threshold. Perhaps the model is too big and overfits to the training dataset. Perhaps the model size is alright, but a rugged loss landscape makes it impossible to find the global minimum using gradient descent methods. But it may also be the case that model is not expressible enough---there is simply no assignment of the model weights that would allow it to compute the underlying true hypothesis. Such a scenario is the point of interest of
	the expressivity of neural networks (and related models), which is an important area of machine learning theory.
	
	For instance, take a simple feed-forward neural network (FFN). It is well known that any Boolean function $f\colon\{0,1\}^n\rightarrow\{0,1\}$ can be computed by a sufficiently large FFN with an appropriate activation function. However, a naive construction requires the size of the network to grow exponentially with the length of the input. Clearly, this is not tight for all Boolean functions. Understanding which model properties (like size or specific activation function) are required to express a given function can guide model design and help explain failure in training.
	
	In this paper we study expressivity of transformers with respect to one specific task, namely, the PARITY task. Let us define PARITY as a function which assigns $0$ to all binary words with an even number of ones and $1$ to all the rest. PARITY in an exemplary task that has been studied both in theoretical computer science (e.g. in the context of circuit complexity ~\citep{furst1984parity,khrapchenko1971complexity}) and in machine learning theory (see e.g.~\cite{regev2009lattices}). One of the reasons why PARITY got considerable attention is because, in a sense, it is the most sensitive Boolean function---flipping a single bit always changes the output of the function. 
	
	In the context of transformers, PARITY was also one of the first formal task for which expressivity has been studied. \cite{hahn2020theoretical} considered a simplified model of transformer with unique hard attention (UHAT) and used random restriction method to prove that PARITY is not computable by a constant-depth UHAT transformer. Later \cite{hao2022formal} strenghten this result by showing that any family of functions computable by a constant-depth UHAT transformer is in the class $AC^0$---which is known to not contain PARITY. Although concerned with a simplified model, these results suggested some explanation to why transformers struggle to generalize on tasks such as PARITY \citet{bhattamishra2020ability}.
	
	Meanwhile, \cite{chiang2022overcoming} gave an explicit construction of a 2-layer transformer with soft attention capable of computing PARITY. This result showed a stark contrast between hard and soft attention mechanisms, albeit with a caveat. The construction of Chiang and Cholak required a 
	positional encoding $i \mapsto i/n$ that depends not only on the position $i$ but also on the input length $n$. This differs from the standard approaches because transformers are usually run on inputs of varying depth. In fact, in NLP transformers often employ causal masking, where the $i$-th position can only access information from the positions with smaller indices (and so, most positions do not 'see' the input length).

	
	This motivated \cite{kozachinskiy2025uniform} to give an alternative 3-layer construction based on a length-independent positional encoding. However, their  construction works just in the the full-attention model and crucially requires that each token accesses the information from all the other positions. Moreover, they gave no bound on how fast their positional encoding has to grow.
	
	\cite{yang2025transformer} gave yet another 2-layer construction that works with causal masking and no positional encoding (except a specialized beginning-of-sequence token). That being said, this construction required two features that differ from standard practice---hard attention and layer normalization with the regularization parameter $\varepsilon$ set to zero. These design choices cannot be used in actual learning via standard gradient methods.
	
	Interestingly, it remains open whether PARITY can be computed by a transformer in just one layer. All the constructions mentioned before use at least 2 layers. In general, proving a lower bound on the number of transformer layers needed to compute some task is hard, especially for the full-attention model~\citep{chen2024theoretical} but some techniques for single layer transformers exist. \cite{sanford2023representational} introduced a communication complexity technique (which works under the assumption of logarithmic precision). More recently, \cite{kozachinskiy2025strassen} gave a new technique, based on a notion of the Split-VC dimension, which works even for transformers with infinite precision. 
	Both techniques require finding a partition of the input bits such that either (a) the communication complexity of the resulting problem, where Alice gets bits from one part of the partition and Bob from the other, is large; or (b) the VC dimension of the concept class, obtained by considering bits from one part of the partition as inputs and from the other as parameters, is large. For PARITY, however, both quantities are constant for any partition, making existing technique not applicable to this function.

	
	\paragraph{Our results.} 
	We solve this open problem and show that no 1-layer transformer can solve PARITY. Our argument relies on the notion of average sensitivity of a function $f$, which is defined as the number of input positions such that flipping them changes the value of $f$, averaged over all input words of the given length. In particular, we show that a function computed by a transformer with a single layer, has the average sensitivity of at most $\sqrt{n} \cdot (\log n)^{O(1)}$. We contrast this with the observation that the average sensitivity of the PARITY grows linearly with the input length.
	
	Our second technical contribution is a new construction of 4-layer transformer that computes PARITY, which improves on the existing constructions in the following aspects: it uses soft attention, works for both full-attention and causal masking architecture, uses only length-independent positional encoding and no layer normalization. Furthermore, its positional encoding is polynomially bounded, which makes it possible to implement on input lengths of practical size.

	\paragraph{Organization of the paper.}
	Section \ref{sec:transformers} introduces all necesary definitions, related to transformers. In Section \ref{sec:lower_bound}, the sensitivity lower bound on 1-layer 1-head transformers is established. Finally, Section \ref{sec:construction} gives our new construction of a transformer for parity.

	\section{Background}
	\label{sec:transformers}
	
	\paragraph{Basic Notation.} 
	For a vector $x \in \mathbb{R}^d$ we write $x_i$ for its $i$-th
	coordinate, and $(\mathbb{R}^d)^*$ for the set of all finite sequences of vectors in
	$\mathbb{R}^d$. 
	We use PARITY to represent the family $\{\mathrm{PARITY}_n\}_{n \geq 1}$,
	where $\mathrm{PARITY}_n(x_1, \ldots, x_n) = x_1 \oplus \cdots \oplus x_n$, where $\oplus$ is the XOR operator, i.e., $a \oplus b = (a + b) \bmod 2$ for
	$a, b \in \{0, 1\}$.

	\paragraph{Attention Layers.}
	We will consider transformers with  two types of attention layers: \emph{full attention} and \emph{causally masked attention}.
	A $d$-dimensional, $H$-head \emph{attention layer} is a
	function $L \colon (\mathbb{R}^d)^* \to (\mathbb{R}^d)^*$ parameterised by query,
	key, and value matrices $Q^{(k)}, K^{(k)}, V^{(k)} \in \mathbb{R}^{d \times d}$,
	with $k = 1, \ldots, H$, a heads-mixing matrix  $W_O \in \mathbb{R}^{d \times dH}$,
	and feed-forward parameters $W_1, W_2 \in \mathbb{R}^{d \times d}$, $b_1, b_2 \in \mathbb{R}^d$.
	On input $(\alpha_1, \ldots, \alpha_n) \in (\mathbb{R}^d)^n$, the layer
	computes for each head $k =1 ,\dots, H$: the \emph{attention weights}  of the $j$th to $i$th position $L^{(k)}_{ij} \in \mathbb{R}$  and \emph{head values}  in $j$th position $h_j^{(k)} \in \mathbb{R}^d$ (for $i, j = 1, \ldots, n$) as follows:
	
	\medskip
	\noindent\begin{minipage}{0.38\textwidth}
		\begin{equation}
			\label{eq_logits}
			L^{(k)}_{ij} = \frac{\langle K^{(k)} \alpha_i, Q^{(k)} \alpha_j \rangle}{\sqrt{d}},
		\end{equation}
	\end{minipage}
	\hfill
	\begin{minipage}{0.45\textwidth}
		\begin{equation}
			\label{eq_headvalues}
			h_j^{(k)} = \frac{\sum_{i = 1}^{m } \exp\{L^{(k)}_{ij}\}\, V^{(k)} \alpha_i}
			{\sum_{i = 1}^{m } \exp\{L^{(k)}_{ij}\}},
		\end{equation}
	\end{minipage}
	\medskip
	
	\noindent where the value of $m$ determines the type of the
	layer: $m=n$ if layer is  \emph{full attention}, and
	$m = j$ if it is \emph{causally masked}. 
	The head values are
	combined and passed through a position-wise feed-forward network:
	
	\medskip
	\noindent\begin{minipage}{0.32\textwidth}
		\begin{equation}
			\label{eq_multihead}
			h_j = W_O \begin{pmatrix} h_j^{(1)} \\ \vdots \\ h_j^{(H)} \end{pmatrix},
		\end{equation}
	\end{minipage}
	\hfill
	\begin{minipage}{0.65\textwidth}
		\begin{equation}
			\label{eq_ffn}
			\beta_j = W_2 \cdot \mathrm{ReLU}\bigl( W_1(h_j + \alpha_j) + b_1 \bigr) + b_2,
		\end{equation}
	\end{minipage}
	\medskip
	
	\noindent  where $\mathrm{ReLU}(x_1, \ldots, x_d) = (\max\{0, x_1\}, \ldots, \max\{0, x_d\})$.
	The output sequence is $(\beta_1, \ldots, \beta_n) = L(\alpha_1, \ldots, \alpha_n)$.
	
	\paragraph{Transformers.}
	We consider transformers that map sequences of tokens into the next, most probable token. 
	A $C$-layer, $H$-head, $d$-dimensional \emph{transformer} over a finite
	vocabulary $\mathcal{V}$ (a set of \emph{tokens} which includes  symbol $\bot$) is a function
	$T \colon \mathcal{V}^* \to \mathcal{V}$ specified by 
	$H$-head $d$-dimensional attention layers
	$L_1, \ldots, L_C$ (all of the same type), an input embedding
	$\mathrm{E} \colon \mathcal{V} \times \mathbb{N}^2 \to \mathbb{R}^d$, and an
	output distribution matrix $W \in \mathbb{R}^{|\mathcal{V}| \times d}$. On input
	$x_1 \dots x_n \in \mathcal{V}^n$, it  applies position-wise the input embedding
	$\alpha_i = \mathrm{E}(x_i, i, n)$, then applies the attention layers computing
	$(\beta_1, \ldots, \beta_n) = L_C \circ \cdots \circ L_1(\alpha_1, \ldots, \alpha_n)$,
	and outputs
	\begin{equation*}
		\label{eq:output_token}
		T(x_1, \ldots, x_n) = \arg\max_{x \in \mathcal{V}}\,
		\bigl(\mathrm{softmax}(W \beta_n)\bigr)_x,
	\end{equation*}
	with the convention that $T(x_1, \ldots, x_n) = \bot$ if the argmax is not unique.

	The embedding $\mathrm{E}$ is of \emph{standard form} if
	$\mathrm{E}(x, i, n) = \mathrm{TE}(x) + \mathrm{PE}(i, n)$ for some \emph{token
		embedding} $\mathrm{TE} \colon \mathcal{V} \to \{0, 1\}^d$ and \emph{positional encoding}
	$\mathrm{PE} \colon \mathbb{N}^2 \to \mathbb{R}^d$; if additionally $\mathrm{PE}(i, n)$  depends only on $i$ (so $\mathrm{PE}(i, n) = g(i)$), we say that  $\mathrm{E}$ is \emph{length-independent}.

	We say that a transformer $T$ \emph{computes} a sequence $\{f_n\}_{n \in \mathbb{N}}$ of
	Boolean functions $f_n \colon \{0, 1\}^n \to \{0, 1\}$ if
	$\{0, 1\} \subseteq \mathcal{V}$ and $T(x) = f_n(x)$, for every $n \in \mathbb{N}$
	and every $x \in \{0, 1\}^n$.

	\paragraph{Sensitivity and Average Sensitivity}
	
	Sensitivity is defined for functions $f\colon\{0, 1\}^n \to \{0, 1\}$ mapping binary strings of  length $n$ to $0$ or $1$. 
	\emph{Sensitivity} of such a function $f$ at input $x\in\{0, 1\}^n$, denoted  $s_x(f)$, is the number of input positions $i\in\{1,\dots,n\}$ such that flipping the $i$th bit of $x$ changes the value of $f(x)$. The \emph{average sensitivity} of $f$ is defined as
	$$
	\label{eq_as}
	as(f) = \sum\limits_{x\in\{0, 1\}^n} \frac{s_x(f)}{2^n},
	$$
	that is, the expected number of bit positions at which $f$ is sensitive, 
	where the expectation is taken uniformly over all inputs $x \in \{0,1\}^n$.
	For example, PARITY on inputs of length $n$ 
	has the maximum possible average sensitivity: since flipping any 
	single bit always changes the parity, we have $s_x(\text{PARITY}) = n$ 
	for every $x \in \{0,1\}^n$, and so, $as(\text{PARITY}) = n$.
	
	In the paper we will also use a technical result establishing a bound on the average sensitivity of 
	polynomial threshold functions.
	A function $f \colon\{0, 1\}^n \to \{0, 1\}$ is called a degree-$d$ \emph{polynomial threshold function} (PTF) if there exists a degree-$d$ polynomial $P \in \mathbb{R}[x_1, \ldots, x_n]$ such that $P(x) > 0$ for all $x\in\{0, 1\}^n$ with $f(x) = 1$, and $P(x) < 0$ for all $x \in\{0, 1\}^n$ with $f(x) = 0$.
	For example, the majority function $\mathrm{maj}_n(x_1,\dots,x_n) = 1$ 
	iff $x_1 + \dots + x_n > \frac{n}{2}$ is a degree-$1$ PTF, witnessed by the 
	polynomial $P(x_1, \dots, x_n) = x_1 + \dots + x_n - \frac{n}{2}$.
	A result of  Kane, that will be important for our paper, establishes an upper bound on the   average sensitivity of $O(1)$-degree PTFs


	\begin{theorem}
		[\citep{kane2014correct}]
		\label{thm_kane}
		Let $f$ be a degree-$d$ PTF over $n > 1$ variables. Then 
		\[as(f) \le \sqrt{n}(\log n)^{O(d\log d)} 2^{O(d^2 \log d)}.\]
	\end{theorem}

	\paragraph{Logic over Reals and Quantifier Elimination}
	
	A first-order logic over the reals, $(\mathbb{R}, <, +, \cdot)$,  is a formalisms allowing 
	to write formulae that are built from real-valued variables and constants, using
	addition ($+$), multiplication ($\cdot$), and comparison ($<$), as well 
	as standard logical connectives ($\lor$, $\land$, $\lnot$, $\to$) and 
	quantifiers ($\forall$, $\exists$). For example, we can write the formula
	$$
	\exists x \; (a \cdot x \cdot x + b \cdot x + c < 0)
	$$
	stating that the quadratic polynomial $a\cdot x^2 + b \cdot x +c$, with coefficients $a, b, c$, has  a 
	negative value for some $x \in \mathbb{R}$. 
	The  seminal Tarski-Seidenberg theorem~\citep{tarski1952decision,seidenberg1954new} implies that $(\mathbb{R}, <, +, \cdot)$ admits quantifier elimination, that is, 
	every formula can be rewritten into an equivalent formula which has no quantifiers. 
	For example, the formula above is equivalent to the quantifier-free formula 
	$$
	b \cdot b - 4 \cdot a \cdot c > 0.
	$$ 
	In this paper, we will use quantifier elimination in $(\mathbb{R}, <, +, \cdot)$ to show that certain properties of 
	transformer computations can be expressed as  polynomial 
	conditions.

	
	\section{One-layer Transformer Cannot Compute PARITY}
	\label{sec:lower_bound}
	
	We start this section with the proof of  the following structural result about Boolean functions, computable by 1-layer transformers.

	\begin{lemma} 
		\label{lemma_ford}
		Assume there is a 1-layer transformer $T$, computing  a sequence of Boolean functions $\{f_n\}_{n = 1}^\infty$. Then
		there exists a first-order formula $\Phi(z_1, \ldots, z_r)$ in the interpretation $(\mathbb{R}, <,+,  \cdot)$  such that for all $n$ there exist $r$  polynomials $l_1, \ldots, l_r \in\mathbb{R}[x_1,\ldots, x_{n}]$ of degree at most 1
		such that for any $x\in\{0, 1\}^{n}$, we have $f_n(x) = 1$ if and only if $\Phi(l_1(x), \ldots, l_r(x)) = 1$.
	\end{lemma}
	
	Then we derive the following upper bound on the average sensitivity of such functions.
	
	\begin{theorem}
		\label{thm:1layer_lower} Assume that a sequence of Boolean functions $\{f_n\}_{n = 1}^\infty$ is computable by a 1-layer transformer. Then $as(f_n) = O(\sqrt{n} \cdot (\log n)^{O(1)})$ as $n\to\infty$.
	\end{theorem}
	
	As PARITY on inputs of length $n$ has average sensitivity exactly $n$, the last result immediately yealds the following corollary.
	
	\begin{corollary}
		No 1-layer transformer computes PARITY.
	\end{corollary}
	
	\begin{remark}
		Note that for 1-layer transformers, there is no difference between the full attention and causally-masked attention models. In the last position, where the output token is computed, we attend all positions in both models. And with just a single layer, the computation of attention in other positions is not affecting the result yet. That is why results of this section are formulated without specifying the type of attention -- full or causal.
	\end{remark}
	
	\begin{remark}
		The upper bound on sensitivity in Theorem  \ref{thm:1layer_lower} is essentially tight as there exists a sequence of Boolean functions, having average sensitivity $\Omega(\sqrt{n})$ and computable by a 1-layer 1-head transformer. For instance, this holds for the sequence of \emph{majority functions}, $\{\mathrm{maj}_n\}_{n\in\mathbb{N}}$, where
		\[\mathrm{maj}_n(x_1, \ldots, x_n) = \begin{cases}1 & x_1 + \ldots + x_n > n/2,\\ 0 & \text{otherwise}.\end{cases}\]
		Almost all inputs to $\mathrm{maj}_n$ have 0 sensitivity, except of  $\Omega(2^n/\sqrt{n})$ inputs from 2 adjacent layers  of the Boolean cube (where $\mathrm{maj}_n$ changes its value) that all have sensitivity $\Omega(n)$. This implies that $as(\mathrm{maj}_n) = \Omega(\sqrt{n})$. On the other hand, a 1-layer 1-head transformer is able to compute this function by computing the expression:
		$\frac{x_1 + \ldots + x_n}{n} - \frac{1}{2n} - \frac{1}{2}$,
		where the first term comes from the average of the input bits with uniform attention weights, and the second term comes from the positional encoding. This quantity is positive for inputs with value of $\mathrm{maj}_n$ equal to 1, and negative for inputs with value of  $\mathrm{maj}_n$ equal to 0. It remains to put this quantity in the output distribution to the token 1, and minus this quantity to the token 0.
	\end{remark}
	
	\subsection{Proof of Lemma \ref{lemma_ford}}
	
	We establish the result under the assumption that the last bit to all functions $f_n$, that is, $x_n$, is always fixed to $0$. This way we get a formula $\Phi_0$, working under this assumption, and also, bu the same argument, one can get a formula $\Phi_1$, working under the assumption that $x_n = 1$. The final formula $\Phi$ is obtained  by 
	\[\Phi = ((z_0 = 0) \to \Phi_0)\land ((z_0 = 1) \to \Phi_1)\]
	where $z_0$ is a fresh free variable, which for inputs of length $n$ gets substituted by the last input bit.
	
	Once again, we assume that $x_n$, the last input bit when the input has length $n$, is fixed to $0$. Thus, the value of $f_n$ becomes the function of the first $n - 1$ input bits, $x\in\{0, 1\}^{n-1}$.

	Assume that the dimension of $T$ is $d$ and the number of attention heads of $T$ is $H$. 
	Our transformer $T$ is 1-layer. Hence, only the attention in the $n$-th position is relevant for computing the output. We first observe that for every $k = 1, \ldots, H$, the vector $h_n^{(k)}\in\mathbb{R}^d$ -- the value of the $k$-th attention head in the $n$-th position -- as a function of $x\in\{0, 1\}^{n-1}$, can be represented as
	\begin{equation}
		\label{eq_linearfrac}
		h_n^{(k)} = \begin{pmatrix}
			\ell_1^{(k)}/\ell_0^{(k)}\ \\ \vdots \\\ell_d^{(k)}/\ell_0^{(k)}
		\end{pmatrix},
	\end{equation}
	where $\ell_0^{(k)}, \ldots,\ell_d^{(k)}\in\mathbb{R}[x_1, \ldots, x_{n-1}]$ are of degree 1. Indeed, by definition:
	\begin{equation}
		\label{eq_hn}
		h_n^{(k)} = \frac{\sum\limits_{i = 1}^n \exp\{L_{in}^{(k)}\} V^{(k)} \alpha_i}{\sum\limits_{i = 1}^n \exp\{L_{in}^{(k)}\}}\in\mathbb{R}^d,
	\end{equation}
	where $V^{(k)}\in\mathbb{R}^{d\times d}$ is a fixed matrix, $\alpha_1 = E(x_1, 1, n)\in\mathbb{R}^d, \ldots, \alpha_{n - 1} = E(x_{n - 1}, n -1, n)\in\mathbb{R}^d, n - 1, n), \alpha_n = E(0, n, n) \in\mathbb{R}^d$ are embeddings of input bits (recall again that $x_n$ is fixed to $0$) and
	\[        \label{eq_logits}
	L^{(k)}_{in} = \langle K^{(k)} \alpha_i, Q^{(k)}\alpha_n\rangle/\sqrt{d}
	\]
	for some fixed matrices $K^{(k)}, Q^{(k)}\in\mathbb{R}^{d\times d}$.
	
	The term in the numerator of \eqref{eq_hn}, corresponding to $i = n$, is a vector $\gamma_{nk} =  \exp\{L_{nn}^{(k)}\}  V^{(k)}\alpha_n\in\mathbb{R}^d$, not depending on $x\in\{0, 1\}^{n - 1}$.  Likewise, the term in the denominator of \eqref{eq_hn} for $i = n$ is a  number $\rho_{nk}\in\mathbb{R}$, not depending on $x\in\{0, 1\}^{n-1}$. In turn, for every $i = 1, \ldots, n - 1$, the $i$-th term in the numerator of \eqref{eq_hn} is a vector:
	\begin{align*}
		&\exp\{L_{in}^{(k)}\}  V^{(k)}\alpha_i = \exp\bigg\{\langle K^{(k)}\alpha_i, Q^{(k)}\alpha_n\rangle/\sqrt{d}\bigg\}  V^{(k)}\alpha_i \\
		&=  \exp\bigg\{\langle K^{(k)}\cdot\mathrm{E}(x_i, i, n), Q^{(k)}\cdot \mathrm{E}(0, n, n)\rangle/\sqrt{d}\bigg\} V^{(k)}\mathrm{E}(x_i, i, n)\in\mathbb{R}^d,
	\end{align*}
	whose value is determined by $i,n,k$, and $x_i$.  
	Hence, it can be written as $(1 - x_i)\gamma^0_{ink} + x_i \gamma^1_{ink}$ for some $\gamma^0_{ink}, \gamma^1_{ink}\in \mathbb{R}^d$. Similarly, the $i$-th term  in the denominator of \eqref{eq_hn}, for $i = 1, \ldots, n - 1$, can be written as $(1 - x_i) \rho^0_{ink} + x_i  \rho^1_{ink}$ for some  $\rho^0_{ink}, \rho^1_{ink}\in \mathbb{R}$. Overall, we obtain:
	\[h_n^{(k)} = \frac{\gamma_{nk} +\sum\limits_{i = 1}^{n - 1}((1 - x_i)\gamma^0_{ink} + x_i \gamma^1_{ink})}{\rho_{nk} + \sum\limits_{i = 1}^{n - 1}((1 - x_i)\rho^0_{ink} + x_i \rho^1_{ink})} = \begin{pmatrix}
		\ell_1^{(k)}/\ell_0^{(k)}\ \\ \vdots \\\ell_d^{(k)}/\ell_0^{(k)}
	\end{pmatrix},\]
	where $\ell_0^{(k)}, \ldots,\ell_d^{(k)}\in\mathbb{R}[x_1, \ldots, x_{n-1}]$ are of degree 1. With this, \eqref{eq_linearfrac} is established.
	
	\medskip
	
	The formula $\Phi_0$ will have $r = (d + 1) H + d$ free variables denoted by $z_{i}^{(k)}$, $i = 0, \ldots, d, k = 1, \ldots, H$ and $a_1, \ldots, a_d$.
	We want the formula $\Phi_0$ to decide, whether the output of the transformer $T$ is equal to the token $1$ on $x\in\{0, 1\}^{n-1}$ when one substitutes
	\[z_i^{(k)}\gets \ell_{i}^{(k)}, \qquad i = 0,\ldots, d, k = 1, \ldots, H,\]
	\[a_i \gets (\alpha_n)_i, \qquad i = 1, \ldots, d.\]
	where  $\ell_{i}^{(k)}\in\mathbb{R}[x_1, \ldots,x_{n-1}]$ are of degree 1 and satisfy \eqref{eq_linearfrac}, and $(\alpha_n)_1, \ldots, (\alpha_n)_d$ are coordinates of the vector $\alpha_n = E(0, n,n)$ (in place of variables $a_1, \ldots, a_d$ we thus substitute just constants, that is, $0$-degree polynomials).
	
	Let $\beta_n\in\mathbb{R}^d$ be the output of the attention layer in the $n$-th position. The transformer, in order to produce its output token, multiplies $\beta_n$ by the output-distribution matrix $W\in\mathbb{R}^{\mathcal{V}\times d}$ and applies the softmax operation to the resulting vector, obtaining a distribution over the vocabulary $\mathcal{V}$:
	\[\mu = \mathrm{softmax}(W \beta_n).\]
	The token with the maximal probability in this distribution is set as the output one. According to our conventions, the maximum has to be unique, and if it is not, $\bot$ is set as the output token. Hence, in order for 1 to be the output token, we must have:
	\[(W\beta_n)_1 > (W\beta_n)_\sigma, \qquad \sigma \in \mathcal{V}\setminus\{1\}. \]
	The latter, when expanded, turns into
	\[W_1^1\beta_n^1 + \ldots + W_d^1\beta^d_n > W_1^\sigma\beta_n^1 + \ldots + W_d^\sigma\beta^d_n, \qquad \sigma \in \mathcal{V}\setminus\{1\}. \]
	Correspondingly, we introduce into our formula $\Phi$ variables $B^1, \ldots, B^d$ that are supposed to correspond to $d$ coordinates of $\beta_n$. We start by writing:
	\begin{equation}
		\label{eq_bb}
		\exists B^1\ldots \exists B^d\,\, \bigwedge_{\sigma\in\mathcal{V}\setminus \{1\}} \bigg(W_1^1B^1 + \ldots + W_d^1B^d > W_1^\sigma B^1 + \ldots + W_d^\sigma B^d\bigg).
	\end{equation}
	This formula is just a conjunction of linear inequalities with some fixed coefficients. These are expressible in $(\mathbb{R}, <, + , \cdot)$ as we are assuming that we can use all real constants in the formulas.

	The vector $\beta_n$, according to \eqref{eq_ffn}, is computed by:
	\[\beta_n = 
	W_2 \cdot  \mathrm{ReLU}\left(W_1(h_n + \alpha_n) + b_1\right) + b_2\in\mathbb{R}^d.\]
	where $W_1, W_2, b_1, b_2$ consist of fixed constants. We already have variables $a_1, \ldots, a_d$, corresponding to coordinates of $\alpha_n$. We introduce variables $\chi_1, \ldots, \chi_d$, corresponding to the coordinates of $h_n$. Our goal is to make sure that the following condition:
	\begin{equation}
		\label{eq_reluhell}
		\exists \chi_1 \ldots \exists \chi_d \,\, \begin{pmatrix}
			B^1 \\ \vdots \\ B^d
		\end{pmatrix} = W_2 \cdot \mathrm{ReLU}\bigg(W_1 \cdot  \begin{pmatrix}
			\chi_1 + a_1 \\ \vdots \\ \chi_d + a_d
		\end{pmatrix} + b_1 \bigg) + b_2    
	\end{equation}
	is expressible in $(\mathbb{R}, <, +, \cdot)$. Indeed, we can introduce a variable for each of the results of intermediate computations in \eqref{eq_reluhell} (with an existential quantifier). Each of these variables is computed  by some linear combinations of other variables, or by $\mathrm{ReLU}$, applied to some other variable. Both are expressible in $(\mathbb{R}, <, +, \cdot)$; the equality $y = \mathrm{ReLU}(x)$ is expressible, for example, via:
	\[(x < 0 \to y = 0) \land (x \ge 0 \to y = x).\] 
	Finally, the vector $h_n$ is computed via 
	\[ h_n = W_O \begin{pmatrix}
		h_n^{(1)}  \\  h_n^{(2)} \\ \vdots \\ h_n^{(H)}
	\end{pmatrix},\]
	and the coordinates of vectors $h_n^{(k)}$ are equal to fractions of degree-1 polynomials $\ell_{i}^{(k)}$ as in \eqref{eq_linearfrac} (and these polynomials are getting substituting in place of variables $z_i^{(k)}$). To finish the construction of $\Phi$, for  $i = 1, \ldots, d$, $k = 1, \ldots, H$, we introduce a variable $\chi_{i}^{(k)}$, corresponding to the $i$-th coordinate of the vector $h_n^{(k)}$. By \eqref{eq_linearfrac}, we have $(h_n)^{(k)}_i = \ell^{(k)}_i/ \ell^{(k)}_0$. Correspondingly, in our formula must include conditions of the form $\chi_i^{(k)} = z^{(k)}_i/z^{(k)}_0$. Writing these conditions using multiplication, we get:
	\begin{equation}
		\label{eq_finalphi}
		\exists \chi_1^{(1)} \ldots \exists \chi_{d}^{(H)} \,\,\begin{pmatrix}
			\chi_1 \\ \vdots \\\chi_d
		\end{pmatrix} = W_O  \begin{pmatrix}
			\chi_1^{(1)} \\ \vdots \\\chi_d^{(1)} \\ \vdots  \\\chi_1^{(H)} \\ \vdots \\\chi_d^{(H)}
		\end{pmatrix} \land\bigg( \bigwedge_{i = 1}^d  \bigwedge_{k = 1}^H z^{(k)}_i = \chi_i^{(k)}\cdot z^{(k)}_0 \bigg).
	\end{equation}
	The formula $\Phi_0$ is obtained as the conjunction of (\ref{eq_bb}--\ref{eq_finalphi}).

	\subsection{Deriving Theorem \ref{thm:1layer_lower} from Lemma \ref{lemma_ford}}
	Take the formula $\Phi(z_1, \ldots, z_r)$, satisfying the conditions of Lemma \ref{lemma_ford}. Due to the Tarski-Seidenberg theorem, one can eliminate quantifiers in it. Thus, without loss of generality, one can assume that $\Phi(z_1, \ldots, z_t)$ is quantifier-free. In other words, this is just a Boolean combination of a fixed number of fixed polynomial  equalities and inequalities  in $z_1, \ldots, z_r$. They turn into a fixed number of $O(1)$-degree polynomial equalities and inequalities in $x_1, \ldots, x_{n}$ when we substitute  $\ell_1, \ldots, \ell_r$ in place of $z_1, \ldots, z_r$. After this substitution, the value of the formula coincides with $f_{n}(x)$. 
	
	In other words, there are $k = O(1)$ polynomials $P_1, \ldots, P_k$ of degree $O(1)$ such that the value of $f_{n}(x)$ for $x\in\{0, 1\}^{n}$ can be computed by asking the  ``extended signs'' of $P_1, \ldots, P_k$ on $x$ (where by the extended sign of a number we mean $0$ if the number is $0$, and $\pm 1$ depending if the number is positive/negative, respectively. We require to be able to say whether a number is exactly 0 in order to handle equalities in the formula $\Phi$).
	
	We now argue that $f_{n}$ must have low average sensitivity, using  Theorem \ref{thm_kane}.

	Average sensitivity is equal to the expected number of positions in a random input such that the value of the function changes after flipping the bit in this position. The value of $f_{n}$ may only change if the extended sign of one of the polynomials $P_1, \ldots, P_k$ changes. It remains to show that a single $O(1)$-degree polynomial $P$ changes its extended sign changes for at most $\sqrt{n} \cdot (\log n)^{O(1)}$ positions on average for a random input.  This follows from Theorem \ref{thm_kane} and the following observation. Take $\varepsilon > 0$ which is smaller in the absolute value than any non-zero value of $P$ on an input from the Boolean cube. Note that $P + \varepsilon, P - \varepsilon$ never take value $0$ on the Boolean cube and thus compute $O(1)$-degree PTFs. By Theorem \ref{thm_kane}, each of $P + \varepsilon, P - \varepsilon$ changes its value from positive to negative or vice versa for at most $\sqrt{n} \cdot (\log n)^{O(1)}$ positions on average. It remains to notice that signs of $P(x) + \varepsilon, P(x) - \varepsilon$ determine the extended sign of $P(x)$. Namely, $P(x) > 0$ only if $P(x) + \varepsilon, P(x) - \varepsilon$ are both positive, $P(x) = 0$ if one of these numbers is positive and the other is negative, and $P(x) < 0$ if both numbers are negative.

	\section{A New Transformer Construction for PARITY}
	\label{sec:construction}
	We require the following fact (a generalization of Faulhaber's formulas to real powers), proved in Appendix  for completeness.
	
	\begin{lemma}
		\label{lem_crazy3}
		For $\alpha \in [5, 100]$, and $n\in\mathbb{N}$, we have
		$1^\alpha + \ldots + n^{\alpha} = \frac{n^{\alpha + 1}}{\alpha + 1} + \frac{n^{\alpha}}{2} + \frac{\alpha n^{\alpha - 1}}{12} +  O(n^{\alpha -2})$.
	\end{lemma}
	
	We now establish our main result of this section.

	\begin{theorem}
		\label{thm_3layer2}
		Both in the full attention and the causally-masked attention models,
		there is a 4-layer transformer for PARITY with a standard-form input embedding, whose positional encoding is length-independent and polynomially bounded (the latter meaning that the $l_\infty$-norm of the positional encoding in position $i$ is bounded by some polynomial in $i$).
	\end{theorem}
	\begin{proof}
		We will need to compute attention just in the last token, from the rest of the tokens we need just positional encoding and input bits. Thus, our construction will work both in the full attention and causally-masked attention models.
		
		Let $x_1 x_2\ldots x_n \in\{0, 1\}^n$ be the input word and let $\Sigma = x_1 + \ldots + x_n$. Let us give a proof  assuming that $1 \le \Sigma \le cn$ for some universal constant $c > 0$ to be defined later. Under this assumption, we will require just 3 layers.
		
		Our general plan is to use attention to obtain the following value at some point:
		\[z=\sum\limits_{i = 1}^n e^{L_{i, n}}(-1)^i / \left( \sum\limits_{i = 1}^n e^{L_{i, n}}\right), \]
		and a guarantee that a) $L_{i,n}$ is maximized at the position $i=\Sigma$; b) $L_{\Sigma,n}$ is much larger than $L_{j,n}$ whenever $j\neq \Sigma$---larger enough to guarantee that $z$ is positive if $\Sigma$ is even and negative otherwise. 

		At the first layer, we compute the weighted sum of inputs bits, where the weight of the positions with 1 is 1, and the weight of the positions with 0 is $\alpha/n$, for some constant $\alpha \in (0, 1)$ to be specified later. Indeed, we can have $\ln n$ at position $n$ from the positional encoding. Thus, we can get attention weights of the form $L_{i,n} = (-\ln(n) + \delta)(1 - x_i)$, where $\delta$ is such that $e^\delta = \alpha$. This will allow us to compute the following expression in the first layer:
		\[\gamma = \frac{10\Sigma}{\Sigma + (\alpha/n) (n - \Sigma)} = \frac{10}{1 + \rho}, \qquad \text{where }\rho = \alpha(1/\Sigma - 1/n).\]
		Note that $0\le \rho \le \alpha < 1$, meaning that $5 \le \gamma \le 10$
		
		At the second layer, using the positional encoding $i\mapsto (\ln i, i^{10})$, and the already computed value of $\gamma$, we can compute the following expression:
		\[\Gamma = \frac{1^{\gamma} \cdot 1^{10} + \ldots + n^{\gamma}\cdot n^{10}}{1^{\gamma} + \ldots + n^{\gamma}}\]
		(using attention weights $L_{i,n} = \ln(i) \cdot  \gamma$).
		\begin{lemma}
			\label{lem_bound}
			\[\Gamma = \tau_n \cdot f(\rho) \cdot \left(1+ O\left(\frac{\rho }{n^2} + \frac{1}{n^3}\right)\right),\]
			where $\tau_n = n^{10}\left(1 + \frac{5}{n} - \frac{5}{3n^2}\right)$ and $f(\rho) = \frac{11 + \rho}{21 + 11\cdot \rho}$.
		\end{lemma}
		\begin{proof}
			Elaborating on the expression for $\Gamma$ with the use of Lemma \ref{lem_crazy3}, since $\gamma, \gamma + 11 \in [5, 100]$, we get:
			\begin{align*}
				\Gamma &= \frac{\frac{n^{\gamma + 11}}{\gamma + 11} \left(1 + \frac{\gamma + 11}{2n} + \frac{(\gamma + 11)(\gamma + 10)}{12n^2} + O\left(\frac{1}{n^3}\right)\right)}{\frac{n^{\gamma + 1}}{\gamma + 1} \left(1 + \frac{\gamma + 1}{2n} + \frac{(\gamma + 1)\gamma}{12n^2} + O\left(\frac{1}{n^3}\right)\right)} \\ &= n^{10}\cdot \frac{\gamma + 1}{\gamma + 11} \cdot \frac{\left(1 + \frac{\gamma + 11}{2n} + \frac{(\gamma + 11)(\gamma + 10)}{12n^2} + O\left(\frac{1}{n^3}\right)\right)}{\left(1 + \frac{\gamma + 1}{2n} + \frac{(\gamma + 1)\gamma}{12n^2} + O\left(\frac{1}{n^3}\right)\right)}.
			\end{align*}
			Observe that $\frac{\gamma + 1}{\gamma + 11} = \frac{\frac{10}{1 + \rho} + 1}{\frac{10}{1 + \rho} + 11} = \frac{11 + \rho}{21 + 11\rho} = f(\rho)$.
			Let us now work separately with the fraction in the last expression. 
			\begin{align*}
				&\frac{\left(1 + \frac{\gamma + 11}{2n} + \frac{(\gamma + 11)(\gamma + 10)}{12n^2} + O\left(\frac{1}{n^3}\right)\right)}{\left(1 + \frac{\gamma + 1}{2n} + \frac{(\gamma + 1)\gamma}{12n^2} + O\left(\frac{1}{n^3}\right)\right)} = \left(1 + \frac{\gamma + 11}{2n} + \frac{(\gamma + 11)(\gamma + 10)}{12n^2} + O\left(\frac{1}{n^3}\right)\right) \cdot \\
				&\left(1 - \frac{\gamma + 1}{2n} - \frac{(\gamma + 1)\gamma}{12n^2} + \frac{(\gamma + 1)^2}{4n^2} +  O\left(\frac{1}{n^3}\right)\right) \\
				&= 1 + \frac{5}{n} + \frac{1}{12n^2}\Big((\gamma + 11)(\gamma +10)-3(\gamma + 1) (\gamma + 11)  - (\gamma + 1) \gamma + 3(\gamma + 1)^2\Big) + O\left(\frac{1}{n^3}\right) \\
				&=  1 + \frac{5}{n} + \frac{80 - 10 \gamma}{12n^2} +  O\left(\frac{1}{n^3}\right) =  1 + \frac{5}{n} + \frac{80 -\frac{100}{1 + \rho} }{12n^2} + O\left(\frac{1}{n^3}\right) \\&=  1 + \frac{5}{n} + \frac{80 -100 + O(\rho) }{12n^2} + O\left(\frac{1}{n^3}\right)=  1 + \frac{5}{n}  - \frac{5}{3n^2} + O\left(\frac{\rho }{n^2} + \frac{1}{n^3}\right) \\
				&=\left(1 + \frac{5}{n}  - \frac{5}{3n^2}\right) \left(1 + O\left(\frac{\rho }{n^2} + \frac{1}{n^3}\right)\right),
			\end{align*}
			and the lemma follows.  \end{proof}

		\begin{lemma}
			\label{lem_aprox}
			Let $f(\rho) =  \frac{11 + \rho}{21 + 11\cdot \rho}$ be the function from Lemma \ref{lem_bound}.
			For $i \in \{1, \ldots, n\}$, define
			\[W_i = -\left(f(\rho) - f(0) - f'(0) \cdot \alpha (1/i - 1/n) \right)^2.\]
			There for all small enough $\alpha \in (0, 1)$, for all $n$ and $\Sigma \in\{1, 2, \ldots, n\}$, we have:
			\begin{itemize}
				\item $W_\Sigma \ge - O(\alpha^4/\Sigma^4)$;
				\item $W_i \le - \Omega\left(\alpha^2 (1/i - 1/\Sigma)^2\right)$ for all $i\neq \Sigma$.
			\end{itemize}
		\end{lemma}
		\begin{proof}
			The function $f(\rho)$ is infinitely differentiable at $(-1, +\infty)$,  meaning that
			\[f(\rho) = f(0) + f'(0)  \rho + O(\rho^2) \qquad \text{ as } \rho \to 0.\]
			Importantly, $f'(0) \neq 0$ as a direct calculation shows that $f'(0) = -100/441$. Recall that $\rho = \alpha(1/\Sigma - 1/n) \le \alpha/\Sigma)$. Thus, $W_\Sigma = -(O(\rho^2))^2 = - O(\alpha^4/\Sigma^4)$.
			In turn, for any $i \neq \Sigma$, we obtain:
			\[-W_i = -\left(f'(0)(\frac{\alpha}{\Sigma} - \frac{\alpha}{i}) + O(\alpha^2/\Sigma^2) \right)^2 =  -\left(\Omega(\alpha(1/i - 1/\Sigma)) + O(\alpha^2/\Sigma^2) \right)^2.\]
			For small enough $\alpha$, the term $\Omega(\alpha(1/i - 1/\Sigma)) = \Omega(\alpha/\Sigma^2)$ dominates the $O(\alpha^2/\Sigma^2)$ term, and the lemma follows.
		\end{proof}
		Our plan now is to devise, at the third layer, attention weights $L_{i,n}$ that are proportional to $W_i$, multiplied by a large factor (so that attention will be mostly concentrated at the position $i = \Sigma$). 

		Note that 
		\[W_i = -(f(\rho) + C/i + A_n)^2\]
		for some absolute constant $C$ and some expression $A_n$, depending only on $n$. Elaborating on this further, we get:
		\[W_i = -2f(\rho) \cdot  (C/i) - (C/i)^2  - 2 A_n \cdot (C/i) + B_{n,\rho}\]
		for some expression $B_{n,\rho}$ that does not depend on $i$.

		We will get attention weights that are very close to  this expression without $B_{n,\rho}$, multiplied by the factor $\tau_n =  n^{10}\left(1 + \frac{5}{n} - \frac{5}{3n^2}\right)$ from Lemma \ref{lem_bound}:
		\begin{align*}
			L_{i, n}' &= \tau_n (W_i - B_{n,\rho}) = \tau_n( -2f(\rho) \cdot  (C/i) - (C/i)^2  - 2 A_n \cdot (C/i)) \\
			&=-2\tau_n \cdot f(\rho) \cdot (C/i) - \tau_n (C/i)^2 -  2 \tau_n A_n \cdot (C/i).
		\end{align*}

		Will not be able to get attention weights exactly $L_{i,n}'$ in  the dot-product attention. The problem is with the term $\tau_n f(\rho)$  which is not yet computed. However, at the $n$-th position we have computed  $\Gamma$, which, by Lemma \ref{lem_bound}, satisfies $\Gamma =  \tau_n \cdot f(\rho) \cdot ( 1 + O(\frac{\rho }{n^2} + \frac{1}{n^3}))$ and thus is really close to $\tau_n\cdot f(\rho)$. In turn, there will be no problem with terms $\tau_n (C/i)^2$ and $2 \tau_n A_n \cdot (C/i)$ as these can be obtained from the dot-product attention using the positional encoding $i \mapsto (1/i, 1/i^2, \tau_i, \tau_i A_i)$. 
		That is, our attention weights at the third layer will be:
		\[L_{i,n} = -2\Gamma \cdot  (C/i)  -\tau_n (C/i)^2 -  2 \tau_n A_n \cdot (C/i).\]
		
		Recall that we assume that $\Sigma \le cn$ for some absolute constant $C> 0$ to be chosen later.
		To finish the proof, we just need to show
		\begin{lemma}
			There exist $\alpha > 0, c > 0$ such that for all large enough $n$ and all $\Sigma$, we have that $L_{\Sigma, n} \ge L_{i,n} + \Omega(n^{6})$ for all $i \neq \Sigma$.
		\end{lemma}
		\begin{proof}
			Note that    \begin{align*}
				L_{i, n} &= -2\tau_n \cdot f(\rho)  (1 + O(\rho/n^2 +1/n^3))\cdot (C/i) - \tau_n (C/i)^2 -  2 \tau_n A_n \cdot (C/i) \\
				&= L_{i,n}' + O(\tau_n \cdot \left(\frac{\rho}{n^2 \cdot i} + \frac{1}{n^3i}\right)) = \tau_n (W_i - B_{n,\rho}) +   O(\tau_n \cdot \left(\frac{\rho}{n^2 \cdot i} + \frac{1}{n^3i}\right)). 
			\end{align*}
			By Lemma \ref{lem_aprox}, for all $\alpha > 0$ small enough we get:
			\begin{equation}
				\label{eq_sigma}
				L_{\Sigma, n} \ge \tau_n B_{n,\rho}- O(\tau_n \alpha^4/\Sigma^4 )  +  O(\tau_n \cdot \left(\frac{\rho}{n^2 \cdot \Sigma} + \frac{1}{n^3\Sigma}\right)) = \tau_n B_{n,\rho} - O(\tau_n E_1) + O(\tau_n E_2),
			\end{equation}
			\begin{equation}
				\label{eq_notsigma}
				L_{i, n} \le \tau_n B_{n,\rho}- \Omega(\tau_n \frac{\alpha^2(i - \Sigma)^2}{i^2 \Sigma^2} )  +  O(\tau_n \cdot \left(\frac{\rho}{n^2 \cdot i} + \frac{1}{n^3i}\right)) = \tau_n B_{n,\rho} - \Omega(\tau_n E_3) + O(\tau_n E_4).
			\end{equation}
			
			We show that by taking $\alpha, c$ to be small enough, we can make  $E_3/E_1, E_3/E_2, E_3/E_4$ arbitrarily large. This will imply that $L_{\Sigma, n}$ is larger by at least $\Omega(\tau_n E_3) = \Omega((1/\Sigma - 1/i)^2\tau_n) = \Omega(\tau_n/\Sigma^4) = \Omega(n^6)$, as required.
			
			We first fix $\alpha$ so that $E_3$ is any given constant time larger than $E_1$. This is possible because $E_3$ is at least $\Omega(\alpha^2/\Sigma^4)$ while $E_1 = O(\alpha^4/\Sigma^4)$. 
			
			We now consider $\alpha$ as fixed. Then $E_3 = \Omega(1/\Sigma^4)$. In turn, $E_2 = O(\left(\frac{\rho}{n^2 \cdot \Sigma} + \frac{1}{n^3\Sigma}\right)) = O(1/(n^2 \Sigma^2))$ (recall that $\rho \le \alpha/\Sigma = O(1/\Sigma)$). Thus, $E_3/E_1 = \Omega(n^2/\Sigma^2)$. Choosing $c$ in $\Sigma \le cn$ small enough makes the fraction $E_3/E_1$ arbitrarily large.
			
			Likewise, considering $E_3/E_4$, since $E_4 = O(\left(\frac{\rho}{n^2 \cdot i} + \frac{1}{n^3i}\right)) = O(1/n^2\Sigma i)$ as $\rho = O(1/\Sigma)$,
			we get up to a fixed constant factor: \[E_3/E_4 \ge ((i - \Sigma)^2/(i^2\Sigma^2)/(1/(n^2\Sigma i)) = (i - \Sigma)^2 \cdot \frac{n^2}{\Sigma i} \ge n/\Sigma,\]
			where the latter is bacause $(i - \Sigma)^2\ge 1, n/i \ge 1$.
			Again, by choosing $c$ sufficiently small, we can make this fraction arbitrarily large.
		\end{proof}
		Hence, the maximum of $L_{i,n}$ is  achieved at $i = \Sigma$, with all the other values being $\Omega(n^{6})$ smaller. We then are able to compute the expression:
		\[z = \sum\limits_{i = 1}^n e^{L_i, n}(-1)^i / \left( \sum\limits_{i = 1}^n e^{L_i, n}\right), \]
		which will be, say, $0.1$-close to $(-1)^\Sigma$. In particular, it will be positive for even $\Sigma$ and negative for odd $\Sigma$. Thus, in the output distribution, it remains to put value $z$ to the token 0, and value $-z$ to the token 1.
		
		Finally, we explain how to get rid of the assumption $0 < \Sigma \le cn$ for some small constant $c > 0$.  We take an even integral number $M > 2/c$. Given an input $x \in \{0, 1\}^n$, in the first layer we compute strings $x^0, \ldots, x^{M - 1}$, where $x^r$ coincides with $x$ on positions $i$ with $i \equiv r \pmod{M}$ and is equal to $0$ elsewhere, except of the position $r + 1$ where it has 1. Thus, we have the following expressions for the bits of $x^r$:
		\[x^r_i = \mathrm{ReLU}\Big(x_i + \mathbb{I}\{i  \equiv r \pmod{M}\}-1\Big)+ \mathbb{I}\{i = r+1\},\]
		which can be computed via FFNs of the first layer (indicators can be taken from the positional encoding). Note that $PARITY(x) = PARITY(x^0) \oplus \ldots \oplus PARITY(x^{M - 1})$ because $M$ is even. Moreover, for each $r = 0, \ldots, M -1$, the number of 1s in $x^r$ is at least 1 and at most $1 + n/M < cn$. Hence, in the next 3 layers we can compute the parities of  $x^0, \ldots, x^{M -1}$ in parallel, using $M$ attention heads and the construction above.  More precisely, we can compute $M$ numbers $z^0, \ldots, z^{M-1}$ such that $z^r$ is $\epsilon$-close to 1 if $PARITY(x^r) = 0$, and $\epsilon$-close to $-1$ if $PARITY(x^r) = 1$ (here $\epsilon>0$ can be made arbitrarily small if $n$ is large enough). Thus, the parity of $x$ will be $0$ if and only if there is an even number of numbers close to $-1$ among $z^0, \ldots, z^{M-1}$. In the FFN of the final layer, it now suffices to sum up expressions of the form $\mathrm{ReLU}(-z^0 -\ldots - z^{M-1} - M + 0.1)$ and all the similar ones where the number of minuses before $z^r$'s is even. This sum will be at least some positive constant if the parity of $x$ is $0$, and 0 otherwise. 
	\end{proof}

	\appendix
	
	\section{Proof of Lemma \ref{lem_crazy}}
	\label{sec:app_crazy}
	
	The proof is via a series of lemmas.
	
	\begin{lemma}
		\label{lem_crazy}
		For $\alpha \in [0, 100]$, and $n\in\mathbb{N}$, we have:
		\[1^\alpha + \ldots + n^{\alpha} = \frac{n^{\alpha + 1}}{\alpha + 1} + O(n^\alpha).\]
	\end{lemma}
	\begin{proof}
		Observe that:
		\[\frac{n^{\alpha + 1}}{\alpha + 1} = \int\limits_0^n x^\alpha dx \le 1^\alpha + \ldots + n^{\alpha} \le  \int\limits_1^{n+1} x^\alpha dx \le \frac{(n+1)^{\alpha + 1}}{\alpha + 1}\]
		(using monotonicity of the function under integral since $\alpha \ge 0$).
		It remains to observe that 
		\begin{align*}
			(n+1)^{\alpha + 1} - n^{\alpha + 1} = n^{\alpha + 1} \left(\left(1 + \frac{1}{n}\right)^{\alpha + 1} - 1 \right) = n^{\alpha + 1} \cdot \left(\frac{\alpha + 1}{n} + O(1/n^2) \right) = O(n^{\alpha}).
		\end{align*}
	\end{proof}
	
	\begin{lemma}
		\label{lem_crazy2}
		For $\alpha \in [2, 100]$, and $n\in\mathbb{N}$, we have:
		\[1^\alpha + \ldots + n^{\alpha} = \frac{n^{\alpha + 1}}{\alpha + 1} + \frac{n^{\alpha}}{2} + O(n^{\alpha -1}).\]
	\end{lemma}
	\begin{proof}
		Observe that:
		\begin{align*}
			\frac{n^{\alpha + 1}}{\alpha + 1} &= \sum\limits_{i = 1}^n \frac{i^{\alpha + 1} - (i-1)^{\alpha+ 1}}{\alpha + 1} =   \sum\limits_{i = 1}^n \frac{i^{\alpha + 1} \cdot \left(1 - \left(1 - \frac{1}{i}\right)^{\alpha + 1}\right)}{\alpha + 1} \\
			&= \sum\limits_{i = 1}^n \frac{i^{\alpha + 1} \left(\frac{\alpha + 1}{i} - \frac{(\alpha + 1)\alpha}{2i^2} + O\left(\frac{1}{i^3}\right)\right)}{\alpha + 1} = \sum\limits_{i = 1}^n i^\alpha - \frac{\alpha}{2}\sum\limits_{i = 1}^n i^{\alpha - 1} + O\left(\sum\limits_{i = 1}^n i^{\alpha - 2}\right)
		\end{align*}
		Using Lemma \ref{lem_crazy} for $\alpha -1$ and $\alpha - 2$, we get:
		\begin{align*}
			1^\alpha + \ldots + n^{\alpha} = \frac{n^{\alpha + 1}}{\alpha +1} + \frac{n^{\alpha}}{2} + O(n^{\alpha - 1}),
		\end{align*}
		as required.
	\end{proof}
	
	We finally get to the proof of Lemma \ref{lem_crazy3}.
	Similarly to the previous proof, we get: 
	\begin{align*}
		\frac{n^{\alpha + 1}}{\alpha + 1} &=  \sum\limits_{i = 1}^n \frac{i^{\alpha + 1} \cdot \left(1 - \left(1 - \frac{1}{i}\right)^{\alpha + 1}\right)}{\alpha + 1} \\ &= \sum\limits_{i = 1}^n \frac{i^{\alpha + 1} \cdot \left(\frac{\alpha + 1}{i} - \frac{(\alpha + 1)\alpha}{2i^2} + \frac{(\alpha + 1)\alpha (\alpha - 1)}{6i^3}+ O\left(\frac{1}{i^4}\right)\right)}{\alpha + 1} \\
		&= S_\alpha  - \frac{\alpha}{2} S_{\alpha - 1} + \frac{\alpha(\alpha -1)}{6} S_{\alpha - 2} + O(S_{\alpha -3}),
	\end{align*}
	where $S_\beta = 1^\beta +\ldots + n^{\beta}$. Using previous lemmas, we get:
	\begin{align*}
		S_\alpha &= \frac{n^{\alpha + 1}}{\alpha +1} + \frac{\alpha}{2} S_{\alpha - 1} -\frac{\alpha(\alpha - 1)}{6} S_{\alpha - 2} + O(S_{\alpha - 3}) \\
		&= \frac{n^{\alpha + 1}}{\alpha +1} + \frac{\alpha}{2} \left(n^{\alpha}/\alpha + n^{\alpha - 1}/2 + O(n^{\alpha - 2})\right) \\&- \frac{\alpha(\alpha - 1)}{6} (n^{\alpha - 1}/(\alpha - 1) + O(n^{\alpha -2})) + O(n^{\alpha -2}) \\
		&= \frac{n^{\alpha + 1}}{\alpha +1} + \frac{n^{\alpha}}{2} + \frac{\alpha n^{\alpha - 1}}{12} + O(n^{\alpha - 2}).
	\end{align*}
	
\end{document}